\documentclass[twoside]{article}
\usepackage{aistats2e}

\usepackage{natbib}

\usepackage{mydefs}
\usepackage{br_risk_bounds}

\begin{document}

\twocolumn[
\aistatstitle{Graph-based Generalization Bounds for Learning Binary Relations}
\aistatsauthor{Ben London \And Bert Huang \And Lise Getoor}
\newcommand{\UMDAddress}{University of Maryland \\ College Park, MD 20742, USA}
\aistatsaddress{
	\UMDAddress \\ \texttt{blondon@cs.umd.edu}
	\And \UMDAddress \\ \texttt{bert@cs.umd.edu}
	\And \UMDAddress \\ \texttt{getoor@cs.umd.edu}
	}
\runningtitle{Graph-based Generalization Bounds for Learning Binary Relations}
\runningauthor{London, Huang, Getoor}
]

\begin{abstract}%
We investigate the generalizability of learned \define{binary relations}: functions that map pairs of instances to a logical indicator. This problem has application in numerous areas of machine learning, such as ranking, entity resolution and link prediction. Our learning framework incorporates an example labeler that, given a sequence $\X$ of $\npt$ instances and a desired training size $\ntr$, subsamples $\ntr$ pairs from $\X \times \X$ without replacement. The challenge in analyzing this learning scenario is that pairwise combinations of random variables are inherently dependent, which prevents us from using traditional learning-theoretic arguments. We present a unified, graph-based analysis, which allows us to analyze this dependence using well-known graph identities. We are then able to bound the generalization error of learned binary relations using Rademacher complexity and algorithmic stability. The rate of uniform convergence is partially determined by the labeler's subsampling process. We thus examine how various assumptions about subsampling affect generalization; under a natural random subsampling process, our bounds guarantee $\tilde\bigO(1/\sqrt{\npt})$ uniform convergence.
\end{abstract}


\section{Introduction}

We investigate the generalizability of a learned \define{binary relation}: a characteristic function $\rel : \XS^2 \to \TwoClass$ that indicates whether the input pair satisfies a relation. For example, if $\rel$ is an equivalence relation, then $\rel(\x,\x') = 1$ indicates that $\x \equiv \x'$; if $\rel$ is a total ordering, then $\rel(\x,\x') = 1$ means that $\x \leq \x'$. Binary relations are found in many learning problems, such as ranking (total ordering), entity resolution (equivalence) and link prediction. There has been significant research in each of these fields individually, yet no unified view of the learning problem. We formulate the learning objective (in \autoref{sec:prelims}) as inductive inference on a product space $\XS^2$, where instances are drawn from an arbitrary distribution over $\XS$. Given a set of $\npt$ \define{independently and identically distributed} (i.i.d.) instances $\X \in \XS^{\npt}$, as well as a subset from $\X^2$ that has been labeled by $\rel$, our goal is to produce a hypothesis $\hyp : \XS^2 \to \TwoClass$ that, with high probability, has low error w.r.t.\ $\rel$.

The primary challenge in analyzing this learning setup is reasoning about the dependency structure of a pairwise training set. Even though the instances are i.i.d., the training examples may not be; since each instance may appear in multiple examples, those that involve a common instance are necessarily dependent. This dependence makes the analysis of generalization nontrivial, since it violates the fundamental assumption of independence used in classical statistics and learning theory, rendering existing results incompatible. When the training set does not simply contain all pairs of instances, 
its dependency structure is difficult to analyze. 
 
In \autoref{sec:graphical_rep}, we introduce graphical representations of the training set and its dependency structure. The training set is viewed as a graph $\Graph$, in which each instance is a vertex and each pairwise example is an (undirected) edge. The dependency structure is represented by the corresponding line graph $\DGraph$, since edges that share a common vertex (instance) are adjacent in $\DGraph$. Casting the dependency structure as a graph allows us to leverage the vast literature of graph theory, which makes our analysis cleaner. The ``amount of dependence" can be quantified by the chromatic number of $\DGraph$, since the chromatic number is the minimal number of independent sets. Using well-known chromatic properties, we are able to upper bound this quantity. We then use graph-based arguments in \autoref{sec:risk_bounds} to bound the generalization error of learned binary relations, using Rademacher complexity or algorithmic stability. We show that the empirical error converges to the expected error at a rate of $\bigO(\sqrt{\rcnt / \ntr})$, where $\ntr$ is the number of labeled training pairs and $\rcnt$ is the maximum frequency of any instance in this set. This ratio depends on how pairs are subsampled (irrespective of their values). Consequently, in \autoref{sec:subsampling}, we explore several learning scenarios in which the subsampling process is working for, against, or is agnostic to the learner (i.e., random). Using a reduction to a random graph model, we prove interesting results for the agnostic scenario. We thus provide a novel analysis of the relationship between subsampling and the rate of uniform convergence for learned binary relations.

\subsection{Related Work}

\citet{goldman:siam93} were the first to address learning a binary relation. Their learning setting is entirely different from ours, in that there is no underlying distribution from which instances are generated, and they are not interested in generalization to unseen data. In their setting, the problem size is polynomial in a finite number of objects. Our learning setting subsumes this one.

Our main analytic tools come from two areas of learning theory: Rademacher complexity and algorithmic stability. The canonical Rademacher bounds are due to \citet{koltchinskii:as02} and \citet{bartlett:jmlr03}, while the canonical stability bounds are due to \citet{bousquet:jmlr02}. Though all of these techniques fundamentally rely on independence, there have been recent applications to non-i.i.d.\ problems. Inspired by \citet{janson:rsa04}, \citet{usunier:nips05} developed a theory of \emph{chromatic} Rademacher analysis, using it to derive generic risk bounds for dependent data, with bipartite ranking as a motivating example. (\citet{ralaivola:jmlr10} used a similar chromatic analysis to derive PAC-Bayes bounds.) Our analysis draws on this work, though we delve deeper into the setting of pairwise prediction. More recently, \citet{mohri:nips08,mohri:jmlr10} used an ``independent blocking" technique to derive both Rademacher- and stability-based risk bounds for time series data drawn from a strongly-mixing stationary process, though this is an entirely different form of dependence.

A number of authors \citep{barhillel:colt03,clemencon:as08,jin:nips09} have developed risk bounds for problems involving pairwise prediction (or learning with pairwise constraints), though they all assume that the training set contains all pairwise combinations of the instances, ignoring the more interesting problem of learning from a (sparse) subsample. Our results are not only more general, but they offer insight into how the subsampling process affects the rate of convergence. \citet{agarwal:jmlr09} consider a random subsampling process similar to one we discuss and derive risk bounds for ranking using algorithmic stability. We provide an alternate analysis that is clean, natural and interpretable.

\section{Preliminaries}
\label{sec:prelims}

This section introduces the necessary background concepts. To clarify certain probabilities, we let $\Pr_{\omega}[\cdot]$ denote the probability of an event w.r.t.\ a random draw of $\omega$ from an implicit sample space $\Omega$. Similarly, let $\Ep_{\omega}[\cdot]$ denote the expectation of a random variable w.r.t.\ a random draw of $\omega \in \Omega$, according to an implicit probability distribution.

\subsection{Problem Setting}
\label{sec:problem_setting}

Let $\XS$ denote an abstract instance space (e.g., $\XS$ could be a subset of Euclidean space). We are interested in learning a binary relation $\rel : \XS^2 \to \TwoClass$. We will limit our analysis to relations that are \define{reflexive}, and either \define{symmetric} or \define{antisymmetric}, examples of which include equivalence and total ordering. 

We define the learning process as follows. We are given a sequence of instances $\X \defeq (\x_1,\dots,\x_{\npt}) \in \XS^{\npt}$, drawn independently and identically from an arbitrary distribution over $\XS$. We are also given access to a \define{labeler} $\Oracle$, a black-box process that returns a subset of pairwise labeled examples. More precisely, given an input $\X$ and a training size $\ntr$, $\Oracle_{\ntr}(\X)$ returns some subset of $\X^2$ (sampled without replacement) that has been labeled according to $\rel$. In practice, this could be a crowd-sourcing application or a targeted surveying process. The subset returned by $\Oracle_{\ntr}(\X)$ is sampled according to a process that is independent of the instance values $\X$. Restricting our theory to consider only labelers that determine \emph{which} pairs to label, independent of $\X$, ensures that the observed data is drawn from the original distribution; otherwise, the labeler could introduce bias. Within this restriction, the subsampling process remains unknown. It may be adversarial or benign, aiming to either weaken or improve generalization; it may also be agnostic, selecting pairs according to some random process. We will return to the topic of subsampling in \autoref{sec:subsampling}.

Due to our assumption of (anti)symmetry, only one example per pair is required, since the converse (or contrapositive) can be inferred from context. This means that $\ntr$ is naturally upper bounded by $\npt \choose 2$. Upon receiving a labeled dataset $\Data \gets \Oracle_{\ntr}(\X)$, we invoke a learning algorithm $\Algo$ to obtain an \define{inductive} hypothesis $\hyp$ from a class $\HS$ of functions from $\XS^2$ to $\TwoClass$. We let $\hyp_{\Data}$ (or, more explicitly, $\hyp_{\Data} \gets \Algo(\Data)$) denote the hypothesis trained on the dataset $\Data$.

\subsection{Generalization}
\label{sec:generalization}

For a given \define{cost function} $\cost : \Reals^2 \to \Reals^+$, define the \define{loss} $\loss$ of a hypothesis $\hyp$ w.r.t.\ a pair $\z \defeq (\x,\x')$ as $\loss(\hyp,\z) \defeq \cost(\rel(\z),\hyp(\z))$. Though there are various cost functions, for learning binary relations, we are most interested in the so-called \define{0-1 loss}, $\lossone(\hyp,\z) \defeq \1[\rel(\z) \neq \hyp(\z)]$.

The quantity we are primarily concerned with is the \define{generalization error}, or \define{risk}. This is the expected loss of a learned hypothesis $\hyp_{\Data}$ w.r.t.\ a random pair $\tilde\z \defeq (\x,\x')$, where $\x$ and $\x'$ are sampled independently and identically from the underlying distribution over $\XS$. We denote this quantity by $\Risk(\hyp_{\Data}) \defeq \Ep_{\tilde\z}[ \loss(\hyp_{\Data},\tilde\z) ]$. In practice, we can estimate the true risk with the \define{training error}, computed as the average loss over all examples in the training set. We denote this by $\Remp(\hyp_{\Data}) \defeq \tfrac{1}{\ntr} \sum_{\z \in \Data} \loss(\hyp_{\Data},\z)$.

Given an empirical risk estimate, we would like to bound its deviation from the true risk. We will refer to this quantity as the \define{defect}, which we denote by $\Defect(\hyp_{\Data}) \defeq \Risk(\hyp_{\Data}) - \Remp(\hyp_{\Data})$. We will use the canonical \define{probably approximately correct} (PAC) framework \citep{valiant:stoc84}, in which we allow the defect to exceed an arbitrary $\eps > 0$ with probability $\fail \in (0,1)$. There are a number of ways of analyzing this probability, of which we will focus on hypothesis complexity and algorithmic stability. Analyses of this nature typically aim to bound
\begin{equation*}
\label{eq:fail_prob}
\Pr\left[ \sup_{\hyp\in\HS} \Defect(\hyp) \geq \eps \right]
~~~\text{or}~~~
\Pr\left[ \sup_{\hyp_{\Data}\gets\Algo_{\Data}} \Defect(\hyp_{\Data}) \geq \eps \right].
\end{equation*}
Solving for $\eps$, one obtains a probabilistic upper bound on the true risk, parameterized by $\fail$.

\section{Graphical Representation}
\label{sec:graphical_rep}

In this section, we analyze the learning problem and its inherent dependencies using a graphical representation.

\subsection{Training Data}
\label{sec:training_data}

Recall that the training set, returned by the labeler $\Oracle$, is an arbitrary subset of $\X^2$ that has been labeled according to $\rel$.
We can represent this as a graph $\Graph \defeq (\Nodes,\Edges)$, in which each vertex $\node_i \in \Nodes$ corresponds to an instance $\x_i$, and each training example $(\x_i,\x_j,\rel(\x_i,\x_j)) \in \Data$ defines an (un)directed edge $(i,j) \in \Edges$. Thus, the subsampling pattern reduces to a graph on the instances $\X$. Note that, if $\rel$ is antisymmetric, then $\Graph$ will be directed; however, since only one example per pair is needed, the corresponding undirected graph is simple (i.e., not a multigraph). This graph representation simplifies the analysis of generalization and allows the labeler's subsampling to be viewed as one of many well-studied graph generation processes.

For example, we can consider an agnostic, random labeler that selects pairs uniformly at random. Using the above graphical representation, this subsampling process can be modeled by the Erd\"{o}s-R\'{e}nyi random graph model $\ERGraph(\npt,\ntr)$. In this model, a graph with $\npt$ vertices and $\ntr$ edges is chosen uniformly from the set of all such graphs. (This model differs slightly from the popular $\ERGraph(\npt,p)$ model, in which each edge is realized independently with probability $p$.) We analyze generalization using various labeler scenarios, including the $\ERGraph(\npt,\ntr)$ labeler, in \autoref{sec:subsampling}.

\subsection{Dependency Structure}
\label{sec:dependence}

For each instance $\x_i$, define a random variable $\X_i$, and recall that these are i.i.d. For each example pair $(\x_i,\x_j)$ found in $\Data$, define a random variable $\Z_{i,j} \defeq (\X_i,\X_j)$. Because each instance may appear in multiple pairs, we have that these random variables are not mutually independent. To make this more concrete, consider the set $\{\Z_{1,2},\Z_{2,3},\Z_{3,4}\}$. Clearly, $\Z_{1,2}$ and $\Z_{2,3}$ are dependent, since they include a common variable, $\X_2$. Now consider variables $\Z_{1,2}$ and $\Z_{3,4}$, and note that $\{1,2\} \cap \{3,4\} = \emptyset$. Observing $(\X_1,\X_2)$ reveals nothing about $(\X_3,\X_4)$; thus, $\bbP(\Z_{3,4} \| \Z_{1,2}) = \bbP(\Z_{3,4})$, and vice versa, so they are mutually independent.

We will represent the dependency structure using a graphical representation due to \citet{erdos:lll75}, known as a \define{dependency graph}.
%
\begin{definition}[Dependency Graph]
\label{def:dep_graph}
Let $\vec\Z \defeq \{\Z_i\}_{i=1}^{\npt}$ be a set of random variables with joint distribution $\bbP(\vec\Z)$, and let $\DGraph \defeq (\Nodes,\Edges)$ be a graph, with $\Nodes \defeq \seq{\npt}$. Then $\DGraph$ forms a \define{dependency graph} w.r.t.\ $\vec\Z$ if every \define{independent set} (i.e., set of non-adjacent vertices) $\Iset \subseteq \Nodes$ satisfies $\bbP( \{\Z_i\}_{i\in\Iset} ) = \prod_{i \in \Iset} \bbP(\Z_i)$.
\end{definition}

In the context of learning binary relations, we construct a dependency graph $\DGraph$ in which each vertex $\node_{i,j}$ is connected to vertices $\{ \node_{k,\ell} : (k=i) \vee (\ell=j) \}$---i.e., the set of vertices that involve instances $\x_i$ or $\x_j$. As a result, any independent set of vertices (in the graph-theoretic sense) is a set of independent random variables. To better understand the structure of this dependency graph, it helps to recall the graph $\Graph$ defined in the previous section; $\DGraph$ is in fact its corresponding \define{line graph}, a graph representing the adjacencies between edges.\footnote{If $\Graph$ is directed, then $\DGraph$ is the line graph of its corresponding undirected graph.} Therefore, every independent set in $\DGraph$ is a \define{matching} in $\Graph$.

\subsection{Chromatic Properties}
\label{sec:graph_coloring}

In this section, we discuss the chromatic properties of the above graphical representations, which will introduce a key lemma used in our generalization bounds. For completeness, we first review some background on graph coloring. For the following definitions, let $\Graph \defeq (\Nodes,\Edges)$ be an arbitrary undirected graph.
%
\begin{definition}[Vertex Coloring]
\label{def:vertex_coloring}
A \define{proper $\ncolors$-vertex-coloring} (often simply referred to as simply a $\ncolors$-coloring) is a mapping from $\Nodes$ to a set of $\ncolors$ color classes, such that no two adjacent vertices have the same color. Equivalently, it is a partitioning $\{ \Color_j : \Color_j \subseteq \Nodes \}_{j=1}^{\ncolors}$, such that $\bigcup_{j=1}^{\ncolors} \Color_j = \Nodes$, $\bigcap_{j=1}^{\ncolors} \Color_j = \emptyset$, and every subset $\Color_j$ is independent. The \define{chromatic number} $\colnum(\Graph)$ is the minimum number of colors needed to achieve a proper coloring.
\end{definition}
%
\begin{definition}[Edge Coloring]
\label{def:edge_coloring}
A \define{proper $\ncolors$-edge-coloring} is a mapping from $\Edges$ to a set of $\ncolors$ color classes, such that no two coincident edges have the same color. The \define{chromatic index} $\colidx(\Graph)$ is the minimum number of colors needed to achieve a proper edge coloring.
\end{definition}
%
\begin{theorem}[\citealp{vizing:coloring64}]
\label{th:vizing}
If $\Graph$ has maximum degree $\mxdeg(\Graph)$, then $\mxdeg(\Graph) \leq \colidx(\Graph) \leq \mxdeg(\Graph) + 1$.
\end{theorem}
For a dependency graph, the chromatic number can be viewed roughly as the ``amount of dependence". The lower the chromatic number, the more independence. If the variables are i.i.d., then the chromatic number is 1. Coloring the vertices of the dependency graph $\DGraph$ described in \autoref{sec:dependence} can be reduced to coloring the edges of the graph $\Graph$ described in \autoref{sec:training_data}, since one is simply the line graph of the other. The chromatic index of $\Graph$ is equal to the chromatic number of $\DGraph$, so bounding one quantity bounds the other. Although there are many graphs for which $\colidx(\Graph) = \mxdeg(\Graph)$, determining the chromatic index of an arbitrary graph is NP-hard, so we will rely on the upper bound. We can therefore state the amount of dependence in terms of the chromatic index of $\Graph$. Note that the maximum degree $\mxdeg(\Graph)$ is simply the maximum frequency of any instance in the training set, which yields the following lemma.
%
\begin{lemma}
\label{lem:chrom_number}
Let $\X \in \XS^{\npt}$ be a set of $\npt$ i.i.d.\ instances, and $\Data \gets \Oracle_{\ntr}(\X)$ an arbitrary training set of size $\ntr$, with maximum instance frequency $\rcnt$. Let $\DGraph$ be the corresponding dependency graph of $\Data$. Then, $\colnum(\DGraph) \leq \rcnt+1$.
\end{lemma}

\section{Generalization Bounds}
\label{sec:risk_bounds}

In this section, we develop risk bounds for learning binary relations using both the Rademacher complexity and algorithmic stability.

\subsection{Concentration Inequalities}

A key component in any generalization analysis is the concentration of random variables. We now present two tail bounds that will be used in our proofs.
%
\begin{theorem}[\citealp{mcdiarmid:sc89}]
\label{th:mcdiarmid}
Let $\vec\Z \defeq \{\Z_i\}_{i=1}^{\npt}$ be a set of i.i.d.\ random variables that take values in $\ZS$. Let $\func : \ZS^{\npt} \to \Reals$ be a measurable function for which there exist constants $\{\lconst_i\}_{i=1}^{\npt}$ such that, for any $i \in \seq{\npt}$, and any inputs $\Data,\Data' \in \ZS^{\npt}$ that differ only in the $i\nth$ variable, $\ab{ \func(\Data) - \func(\Data') } \leq \lconst_i$.
Then, for any $\eps > 0$,
\begin{equation*}
\label{eq:tail_bound}
\Pr\left[ \func(\vec\Z) - \Ep[\func(\vec\Z)] \geq \eps \right]
	\leq \exp\( \frac{-2\eps^2}{\sum_{i=1}^{\npt}\lconst_i^2} \).
\end{equation*}
\end{theorem}
%
\begin{theorem}[\citealp{usunier:nips05}]
\label{th:chrom_tail_bound}
Let $\DGraph$ be a dependency graph (\autoref{def:dep_graph}) for a set of random variables $\vec\Z \defeq \{\Z_i\}_{i=1}^{\npt}$ that take values in $\ZS$. Let $\{\vec\Z_j\}_{j=1}^{\colnum(\DGraph)}$ denote the subsets induced by an optimal proper coloring of $\DGraph$, and let $\npt_j \defeq \card{\vec\Z_j}$. Finally, let $\func : \ZS^{\npt} \to \Reals$ be a measurable function where: (a) there exist functions $\{ \func_j : \ZS^{\npt_j} \to \Reals \}_{j=1}^{\colnum(\DGraph)}$, such that $\func(\vec\Z) = \sum_j \func_j(\vec\Z_j)$; (b) there exists a constant $\lconst$ such that every $\func_j$ is $\lconst$-Lipschitz w.r.t.\ the Hamming metric. Then, for any $\eps > 0$,
\begin{equation*}
\label{eq:tail_bound}
\Pr\left[ \func(\vec\Z) - \Ep[\func(\vec\Z)] \geq \eps \right]
	\leq \exp\( \frac{-2\eps^2}{\colnum(\DGraph) \npt\lconst^2} \).
\end{equation*}
\end{theorem}

\subsection{Rademacher Complexity}
\label{sec:rademacher}
 
Informally, the Rademacher complexity measures a hypothesis class' expressive power, quantified by its ability to fit a random signal.
We slightly adapt the traditional definition to better suit our learning context.
%
\begin{definition}[Rademacher Complexity]
\label{def:rademacher}
Let $\XS$ be an instance space, and $\ZS$ some alternate space. Let $\map : \XS^{\npt} \to \ZS^{\ntr}$ denote a mapping from $\npt$ instances from $\XS$ to $\ntr$ instances from $\ZS$. Let $\radvars$ be a set of \define{Rademacher variables} that independently take values in $\TwoClass$ with equal probability. For a class $\FS$ of functions from $\ZS$ to $\Reals$, define the \define{empirical Rademacher complexity}
of $\FS \circ \map$, w.r.t.\ instances $\X \in \XS^{\npt}$, as
\begin{equation*}
\label{eq:emp_rademacher}
\Rad_{\X}(\FS \circ \map)
	\defeq \Ep_{\radvars}\left[ \frac{2}{\ntr} \sup_{\func\in\FS} \ab{\sum_{i=1}^{\ntr} \radv_i \func(\z_i)} \right].
\end{equation*}
Define the \define{Rademacher complexity} of $\FS \circ \map$, w.r.t.\ i.i.d.\ samples of size $\npt$, as $\Rad_{\npt}(\FS \circ \map) \defeq \Ep_{\X}[ \Rad_{\X}(\FS \circ \map) ]$.
\end{definition}
In our scenario, the data may exhibit dependencies; yet, as we will see, this does not affect the traditional symmetry argument used in Rademacher analysis.
%
\begin{theorem}
\label{th:rad_generic_bound}
Let $\X \in \XS^{\npt}$ be an i.i.d.\ sample of $\npt$ instances, and $\Data \gets \Oracle_{\ntr}(\X)$ an arbitrary training set of size $\ntr$, with maximum instance frequency $\rcnt$. Let $\FS$ be a class of functions from $\ZS \defeq \XS^2$ to $[0,1]$. Then, for any $\npt \geq 2$, any $\ntr \geq 1$, any $\func \in \FS$, and any $\fail \in (0,1)$, with probability at least $1 - \fail$ over draws of $\X$,
\begin{equation}
\label{eq:rad_generic_bound}
\Ep_{\tilde\z}[\func(\tilde\z)]
	\leq \frac{1}{\ntr}\sum_{\z\in\Data}\func(\z)
	+ \Rad_{\npt}(\FS\circ\Oracle_{\ntr}) + \sqrt{\frac{\rcnt+1}{2\ntr}\ln\frac{1}{\fail}}.
\end{equation}
\end{theorem}
\begin{proof}
With $\DGraph$ as the dependency graph of $\Data$, we invoke an optimal proper coloring, which partitions $\Data$ into $\colnum(\DGraph) \leq \rcnt+1$ (via \autoref{lem:chrom_number}) independent sets. We can consequently express the defect $\Defect(\func) \defeq \Ep_{\tilde\z}[\func(\tilde\z)] - \frac{1}{\ntr}\sum_{\z\in\Data}\func(\z)$ as a sum over functions of independent random variables. It is easy to show that each of these functions is $(1/\ntr)$-Lipschitz w.r.t.\ the Hamming metric. We therefore apply \autoref{th:chrom_tail_bound} and obtain
\begin{equation}
\label{eq:fail_prob_rad_defect}
\Pr_{\X}\left[ \sup_{\func\in\FS}\Defect(\func) - \Ep_{\X}[\sup_{\func\in\FS} \Defect(\func)] \geq \eps \right]
	\leq \exp\( \frac{-2 \eps^2 \ntr}{\rcnt+1} \),
\end{equation}
To bound $\Ep_{\X}[\sup_{\func\in\FS} \Defect(\func)]$, we start by imagining a ``ghost sample" $\X'$ of $\npt$ i.i.d.\ instances that have been labeled using the same pattern as $\Data$ to create $\Data'$. 
Note that $\Ep_{\tilde\z}[\func(\tilde\z)] = \Ep_{\X'}\left[ \frac{1}{\ntr}\sum_{\z'\in\Data'}\func(\z') \right]$ via linearity of expectation, since all $\z'$ have the same marginal distribution.
Using the ghost sample and Jensen's inequality, we have that
\begin{align}
\label{eq:exp_defect_bound1}
&\Ep_{\X}[\sup_{\func\in\FS} \Defect(\func)]
	= \Ep_{\X}\left[ \sup_{\func\in\FS} \Ep_{\tilde\z}[\func(\tilde\z)] - \frac{1}{\ntr}\sum_{\z\in\Data}\func(\z) \right] \nonumber\\
	&~~~~= \Ep_{\X}\left[ \sup_{\func\in\FS} \Ep_{\X'}\left[ \frac{1}{\ntr}\sum_{\z'\in\Data'}\func(\z') \right] - \frac{1}{\ntr}\sum_{\z\in\Data}\func(\z) \right] \nonumber\\
	&~~~~\leq \Ep_{\X,\X'}\left[ \sup_{\func\in\FS} \frac{1}{\ntr}\sum_{i=1}^{\ntr} \func(\z_i') - \func(\z_i) \right].
\end{align}
Now define a set of Rademacher variables $\radvars$, and define random variables $\{\Z_i\}_{i=1}^{\ntr}$ and $\{\Z_i'\}_{i=1}^{\ntr}$ for $\Data$ and $\Data'$ sets respectively. Because they are labeled using the same pattern, and have isomorphic dependency graphs, we have that $\bbP(\Z_1,\dots,\Z_{\ntr}) = \bbP(\Z_1',\dots,\Z_{\ntr}')$. In fact, if we exchange any $\Z_i$ and $\Z_i'$, we have that $\bbP(\Z_1,\dots,\Z_i',\dots,\Z_{\ntr}) = \bbP(\Z_1',\dots,\Z_i,\dots,\Z_{\ntr}')$. Thus, since every draw of $\radvars$ occurs with equal probability, we have via symmetry that
\begin{align*}
&\text{Eq. (\ref{eq:exp_defect_bound1})} ~
	\leq \Ep_{\X,\X',\radvars}\left[ \sup_{\func\in\FS} \frac{1}{\ntr} \sum_{i=1}^{\ntr} \radv_i(\func(\z_i') - \func(\z_i)) \right] \\
	&~~~~\leq \Ep_{\X,\radvars}\left[ \sup_{\func\in\FS} \frac{2}{\ntr} \ab{\sum_{i=1}^{\ntr} \radv_i \func(\z_i')} \right]
	\leq \Rad_{\npt}(\FS\circ\Oracle_{\ntr}).
\end{align*}
and so $\Ep_{\X}[\sup_{\func\in\FS} \Defect(\func)] \leq \Rad_{\npt}(\FS\circ\Oracle_{\ntr})$. 
To obtain \autoref{eq:rad_generic_bound}, we simply set \autoref{eq:fail_prob_rad_defect} equal to $\fail$ and solve for $\eps$.
%
\end{proof}
It is possible to derive a similar risk bound for the empirical Rademacher complexity by applying \autoref{th:chrom_tail_bound} to the difference of $\Rad_{\npt}(\FS\circ\Oracle_{\ntr}) - \Rad_{\X}(\FS\circ\Oracle_{\ntr})$. We omit this proof to save space, since the remainder of our work does not require the empirical Rademacher complexity.

To make these bounds more functional, we will replace $\Rad_{\npt}$ with an empirically verifiable quantity. For certain hypothesis classes, it is possible to show that the Rademacher complexity is bounded by a function of the model parameters.
One such class of hypotheses is \define{reproducing kernel Hilbert spaces} (RKHS), which subsume the popular support vector machine (SVM) \citep{cristianini:svms}. Formally, for some mapping $\Kmap : \ZS \to \tilde\ZS$, where $\ZS$ is an instance space and $\tilde\ZS$ is a Hilbert space, endowed with an inner product $\langle \cdot,\cdot \rangle$, a kernel $\K : \ZS^2 \to \Reals$ is a function such that, for all $(\z,\z') \in \ZS^2$, $\K(\z,\z') = \langle \Kmap(\z),\Kmap(\z') \rangle$. The only requirement is that the kernel's Gram matrix $\Kmat : \Kmat_{i,j} = \K(\z_i,\z_j)$ be symmetric and positive semidefinite. RKHS hypotheses are generally of the form $\hyp_{\K,\Data}(\z) \defeq \sum_{\z'\in\Data}^{\ntr} \K(\z,\z')$, where $\Data \in \ZS^{\ntr}$ is a set of reference points from the problem domain (e.g., support vectors). We may reasonably assume that the norm of the kernel mapping is uniformly bounded by a constant $\Knrm$; i.e., the mapped data is contained within a ball of radius $\Knrm$. We denote the class of kernel functions by $\HS_{\K}$. Borrowing results from \citet{koltchinskii:as02} and \citet{bartlett:jmlr03}, we are able to prove the following risk bounds for kernel-based hypotheses with bounded kernels, in the context of learning binary relations.

For the following, we use the \define{ramp loss} as a surrogate for the 0-1 loss. For a given $\conf > 0$, a real-valued hypothesis $\hyp : \XS^2 \to \Reals$ and an example $\z$, define the ramp loss as $\lossramp(\hyp,\z) \defeq \min\{ \max\{ 0, 1 - \rel(\z) \hyp(\z) / \conf \}, 1 \}$. To differentiate this from the 0-1 loss when dealing with risk metrics, we henceforth use a superscript $\1$ or $\conf$.
%
\begin{theorem}
\label{th:rad_kernel_risk_bound}
Let $\X \in \XS^{\npt}$, $\Data \gets \Oracle_{\ntr}(\X)$ and $\rcnt$ be as defined in \autoref{th:rad_generic_bound}. Let $\hyp_{\K,\Data}$ be a RKHS hypothesis trained on $\Data$, such that $\sup_{\z} \norm{\Kmap(\z)} \leq \Knrm$. Then, for any $\npt \geq 2$, any $\ntr \geq 1$, any $\conf > 0$, and any $\fail \in (0,1)$, with probability at least $1 - \fail$ over draws of $\X$,
\begin{equation}
\label{eq:rad_kernel_risk_bound}
\Risk^{\1}(\hyp_{\K,\Data})
	\leq \Remp^{\conf}(\hyp_{\K,\Data}) + \frac{4\Knrm}{\conf\sqrt{\ntr}} + \sqrt{\frac{\rcnt+1}{2\ntr}\ln\frac{1}{\fail}}.
\end{equation}
\end{theorem}
\begin{proof}
Note that $\lossramp$ dominates $\lossone$, and thus $\Risk^{\1}(\hyp_{\K,\Data}) \leq \Risk^{\conf}(\hyp_{\K,\Data})$. Since the ramp loss is bounded by $[0,1]$, we apply \autoref{th:rad_generic_bound} and obtain
\begin{align*}
\Risk^{\1}(\hyp_{\Data})
	&\leq \Remp^{\conf}(\hyp_{\Data}) + \Rad_{\npt}(\lossramp\!\circ\!\HS_{\K}\!\circ\!\Oracle_{\ntr}) + \sqrt{\frac{\rcnt+1}{2\ntr}\ln\frac{1}{\fail}}.
\end{align*}
To bound $\Rad_{\npt}(\lossramp\circ\HS_{\K}\circ\Oracle_{\ntr})$, we use Talagrand's contraction lemma \citep{ledoux:probability91} and borrow a result from \citet[Lemma 22]{bartlett:jmlr03}, from which we obtain
\begin{equation*}
\Rad_{\npt}(\lossramp\circ\HS_{\K}\circ\Oracle_{\ntr})
	\leq \frac{2}{\conf} ~ \Rad_{\npt}(\HS_{\K}\circ\Oracle_{\ntr})
	\leq \frac{2}{\conf} \cdot \frac{2\sqrt{\tr(\Kmat)}}{\ntr}.
\end{equation*}
Using Cauchy-Schwarz, we can bound the trace of the kernel's Gram matrix as
\begin{equation*}
\tr(\Kmat)
	= \sum_{\z\in\Data} \langle \Kmap(\z), \Kmap(\z) \rangle
	\leq \sum_{\z\in\Data} \norm{\Kmap(\z)}^2
	\leq \ntr \Knrm^2.
\end{equation*}
We therefore have that $\Rad_{\npt}(\lossramp\circ\HS_{\K}\circ\Oracle_{\ntr}) \leq \frac{4\Knrm}{\conf\sqrt{\ntr}}$. 
\end{proof}
%
Note that this analysis slightly improves upon that of \citet{usunier:nips05} in that we use the regular Rademacher complexity instead of their so-called \define{fractional Rademacher complexity}. Because of this, our Rademacher term is $\bigO(1/\sqrt{\ntr})$, compared to $\bigO(\sqrt{\rcnt/\ntr})$ using the fractional Rademacher complexity; note that $1/\sqrt{\ntr} \leq \sqrt{\rcnt/\ntr}$ for all $\rho \geq 1$.

\subsection{Algorithmic Stability}
\label{sec:stability}

This section derives a different generalization bound for learning pairwise relations using our previous graph representations and \emph{algorithmic stability}.
%
\begin{definition}[Uniform Stability]
\label{def:uniform_stability}
For a training set $\Data$, let $\Data'$ be a duplicate of $\Data$ with the $i\nth$ example removed. A learning algorithm $\Algo$ has \define{uniform stability} $\stab$ w.r.t.\ a loss function $\loss$ if, for any $\Data \in \ZS^{\ntr}$, and any $i \in \seq{\ntr}$, $\Algo$ returns hypotheses $\hyp_{\Data} \gets \Algo(\Data)$ and $\hyp_{\Data'} \gets \Algo(\Data')$ such that
\begin{equation*}
\label{eq:uniform_stability}
\sup_{\tilde\z \in \ZS} \ab{ \loss(\hyp_{\Data},\tilde\z) - \loss(\hyp_{\Data'},\tilde\z) } \leq \stab.
\end{equation*}
\end{definition}
In other words, excluding any single example from training will increase the loss, w.r.t.\ any test example $\tilde\z$, by at most $\stab$. Of course, $\stab$ must be a function of the size of the training set; indeed, we will later show that generalization is only possible when $\stab = \bigO(1/\ntr)$. To highlight this dependence, we henceforth use the notation $\stab_{\ntr}$. Using this notion of stability, we now derive alternate risk bounds for learning binary relations.
%
\begin{lemma}[\citealp{bousquet:jmlr02}]
\label{lem:defect_lipschitz}
Let $\Algo$ be a learning algorithm with uniform stability $\stab$ w.r.t.\ a loss function $\loss$, where $\loss$ is upper bounded by $\costbound$. Then, for any $i \in \seq{\ntr}$, and any training sets $\Data,\Data' \in \ZS^{\ntr}$ that differ only in the value of the $i\nth$ example, $\Algo$ returns hypotheses $\hyp_{\Data} \gets \Algo(\Data)$ and $\hyp_{\Data'} \gets \Algo(\Data',\Graph)$ such that
\begin{equation*}
\label{eq:defect_lipschitz}
\ab{ \Defect(\hyp_{\Data}) - \Defect(\hyp_{\Data'}) } \leq 4\stab_{\ntr} + \frac{\costbound}{\ntr}.
\end{equation*}
\end{lemma}
%
\begin{theorem}
\label{th:stab_risk_bound}
Let $\X \in \XS^{\npt}$ be an i.i.d.\ sample of $\npt$ instances, and $\Data \gets \Oracle_{\ntr}(\X)$ an arbitrary training set of size $\ntr$, with maximum instance frequency $\rcnt$. Let $\Algo$ be a learning algorithm with uniform stability $\stab$ w.r.t.\ a loss function $\loss$ (upper bounded by $\costbound$), and let $\hyp_{\Data} \gets \Algo(\Data)$. Then, for any $\npt \geq 2$, any $\ntr \geq 1$, and any $\fail \in (0,1)$, with probability at least $1 - \fail$ over draws of $\X$,
\begin{equation}
\label{eq:stab_risk_bound}
\Risk(\hyp_{\Data})
	\leq \Remp(\hyp_{\Data}) + 4\rcnt\stab_{\ntr} + (4\ntr\stab_{\ntr} + \costbound) \sqrt{\frac{\rcnt}{\ntr}\ln\frac{1}{\fail}}.
\end{equation}
\end{theorem}
\begin{proof}
We begin by showing that the defect satisfies the conditions of McDiarmid's inequality (\autoref{th:mcdiarmid}). By \autoref{lem:defect_lipschitz}, replacing any single \emph{example} will change the defect by at most $4\stab_{\ntr} + \costbound/\ntr$. However, if we replace any single \emph{instance} $\x_i$, this will affect up to $\rcnt_i$ examples, where $\rcnt_i$ denotes the frequency of $\x_i$ in the training set. We therefore have that replacing any $\x_i$ has Lipschitz constant $\lconst_i = \rcnt_i (4\stab_{\ntr} + \costbound/\ntr)$. Since the instances are i.i.d., we can apply McDiarmid's inequality and obtain
\begin{align*}
\label{eq:fail_prob_stab_defect}
\lefteqn{\Pr_{\X}\left[ \Defect(\hyp_{\Data}) - \Ep_{\X}[\Defect(\hyp_{\Data})] \right]} ~~~~~~~~ \\
	&\leq \exp\( \frac{-2\eps^2}{\sum_{i=1}^{\npt} \rcnt_i^2 (4\stab_{\ntr} + \costbound/\ntr)^2} \) \\
	&\leq \exp\( \frac{-2\eps^2\ntr^2}{(4\ntr\stab_{\ntr} + \costbound)^2 \rcnt \sum_{i=1}^{\npt} \rcnt_i} \) \\
	&= \exp\( \frac{-\eps^2\ntr}{(4\ntr\stab_{\ntr} + \costbound)^2 \rcnt} \).
\end{align*}
The last line is due to the \define{handshaking lemma}, which states that the sum of the degrees (i.e., instance frequencies) in a graph is equal to twice the number of edges (i.e., examples). Setting the above equal to $\fail$ and solving for $\eps$, we have that
\begin{equation*}
\Risk(\hyp_{\Data})
	\leq \Remp(\hyp_{\Data}) + \Ep_{\X}[\Defect(\hyp_{\Data})] + (4\ntr\stab_{\ntr} + \costbound) \sqrt{\frac{\rcnt}{\ntr}\ln\frac{1}{\fail}}.
\end{equation*}

What remains is to upper bound the expected defect. Using linearity of expectation, we can state this as
\begin{equation*}
\Ep_{\X}[ \Defect(\hyp_{\Data}) ]
	= \frac{1}{\ntr} \sum_{i=1}^{\ntr} \Ep_{\X,\tilde\z}[ \loss(\hyp_{\Data},\tilde\z) - \loss(\hyp_{\Data},\z_i) ].
\end{equation*}
Note that example $\z_i = (\x,\x')$ depends on any examples that share either $\x$ or $\x'$, i.e., its neighborhood $\neigh(\z_i)$ in the dependency graph. However, by removing $\z_i$ and $\neigh(\z_i)$ from the training set, $\z_i$ becomes independent of any of the remaining examples. Accordingly, let $\Data' \defeq \Data \setminus \{ \z_i, \neigh(\z_i) \}$, and let $\hyp_{\Data'}$ be the resulting hypothesis. Because we have removed $\deg(\z_i) + 1$ examples from training, by uniform stability, we pay a penalty of at most $\stab_{\ntr}(\deg(\z_i) + 1)$ loss per prediction. Further, recall that $\DGraph$ is the line graph of the graph $\Graph$ defined in \autoref{sec:training_data}. It is therefore easy to show that any node in $\DGraph$ (i.e., edge in $\Graph$) has degree equal to the sum of the degrees of its endpoints in $\Graph$, minus two; hence, $\deg(\z_i) + 1 = \deg(x) + \deg(x') - 1 \leq 2\rcnt -1$. We therefore have that
\begin{align*}
\Ep_{\X}[ \Defect(\hyp_{\Data}) ]
	&\leq \frac{1}{\ntr} \sum_{i=1}^{\ntr} \Ep_{\X,\tilde\z}[ \loss(\hyp_{\Data'},\tilde\z) - \loss(\hyp_{\Data'},\z_i) ] \\
	&~~~~ + 2\stab_{\ntr}(\deg(\z_i) + 1) \\
	&= \frac{2\stab_{\ntr}}{\ntr} \sum_{i=1}^{\ntr} 0 + \deg(\z_i) + 1 \\
	&\leq 2 \stab_{\ntr} (2\rcnt-1)
	\leq 4 \rcnt \stab_{\ntr},
\end{align*}
where the second line follows from symmetry, since $\z$ and $\z_i$ are now i.i.d.\ variables.
\end{proof}
%
To obtain a non-vacuous bound, we require that $\stab_{\ntr} = \bigO(1/\ntr)$. This precludes stability w.r.t.\ the 0-1 loss $\lossone$, since any algorithm will have either $\stab = 0$ or $\stab = 1$, regardless of $\ntr$. We therefore use the ramp loss $\lossramp$ again. To use the ramp loss, we introduce the notion of \define{classification stability}. For the following, we consider learning algorithms that output a real-valued hypothesis $\hyp : \ZS \to \Reals$, where $\sgn(\hyp(\z))$ is the predicted label of $\z$.
%
\begin{definition}[Classification Stability]
\label{def:class_stability}
Let $\Data$ and $\Data'$ be as defined in \autoref{def:uniform_stability}. A learning algorithm $\Algo$ has \define{classification stability} $\stab$ if, for any $\Data \in \ZS^{\ntr}$, and any $i \in \seq{\ntr}$, $\Algo$ returns real-valued hypotheses $\hyp_{\Data} \gets \Algo(\Data)$ and $\hyp_{\Data'} \gets \Algo(\Data')$ such that
\begin{equation*}
\label{eq:class_stability}
\sup_{\tilde\z \in \ZS} \ab{ \hyp_{\Data}(\tilde\z) - \hyp_{\Data'}(\tilde\z) } \leq \stab.
\end{equation*}
\end{definition}
%
\begin{lemma}[\citealp{bousquet:jmlr02}]
\label{lem:ramp_loss_stability}
A learning algorithm with classification stability $\stab$ has uniform stability $\stab / \conf$ w.r.t.\ the ramp loss $\lossramp$.
\end{lemma}
%
\begin{theorem}
\label{th:stab_ramp_risk_bound}
Let $\X \in \XS^{\npt}$, $\Data \gets \Oracle_{\ntr}(\X)$ and $\rcnt$ be as defined in \autoref{th:stab_risk_bound}. Let $\Algo$ be a learning algorithm with classification stability $\stab$, and let $\hyp_{\Data} \gets \Algo(\Data)$. Then, for any $\npt \geq 2$, any $\ntr \geq 1$, any $\conf > 0$, and any $\fail \in (0,1)$, with probability at least $1 - \fail$ over draws of $\X$,
\begin{equation}
\label{eq:stab_ramp_risk_bound}
\Risk^{\1}(\hyp_{\Data})
	\leq \Remp^{\conf}(\hyp_{\Data}) + \frac{4\rcnt\stab_{\ntr}}{\conf} + \( \frac{4\ntr\stab_{\ntr}}{\conf} + 1 \) \sqrt{\frac{\rcnt}{\ntr}\ln\frac{1}{\fail}}.
\end{equation}
\end{theorem}
\begin{proof}
The proof follows directly from \autoref{th:stab_risk_bound} and \autoref{lem:ramp_loss_stability}, with the ramp loss obviously upper bounded by $\costbound = 1$.
\end{proof}
%
The application of this bound still depends on a stability parameter, which is unique to the learning algorithm. As a demonstration, we will focus on the class of kernel methods described in \autoref{sec:rademacher}; specifically, SVM classification. Recall that, using stability analysis, generalization is made possible by properties of the learning algorithm, not the complexity of the hypothesis class. For the class of kernel methods in particular, this mechanism is \define{regularization}. We define a kernel-based regularization algorithm as one of the form
\begin{equation*}
\label{eq:kernel_algo}
\Algo_{\K}(\Data) \defeq \argmin_{\hyp\in\HS_{\K}} \frac{1}{\ntr} \sum_{\z\in\Data} \loss(\hyp,\z) + \reg\norm{\hyp}^2,
\end{equation*}
where $\reg > 0$ is a regularization parameter. The loss function varies, depending on the application and algorithm. In SVM classification, it is common to minimize the \define{hinge loss}, defined as $\losshinge(\hyp,\z) \defeq \max\{ 0, 1 - \rel(\z) \hyp(\z) \}$. Denote the empirical hinge risk by $\Remp^{\hinge}$. Using a stability result from \citet{bousquet:jmlr02}, we obtain a risk bound for SVM classification.
%
\begin{lemma}[\citealp{bousquet:jmlr02}]
\label{lem:kernel_class_stability}
An SVM learning algorithm, with $\sup_{\z} \norm{\Kmap(\z)} \leq \Knrm$ and regularization parameter $\reg > 0$, has classification stability $\stab_{\ntr} \leq \Knrm^2 / (2\reg\ntr)$.
\end{lemma}
%
\begin{theorem}
\label{th:stab_svm_risk_bound}
Let $\X \in \XS^{\npt}$, $\Data \gets \Oracle_{\ntr}(\X)$ and $\rcnt$ be as defined in \autoref{th:stab_risk_bound}. Let $\Algo_{\K}$ be an SVM learning algorithm, with $\sup_{\z} \norm{\Kmap(\z)} \leq \Knrm$ and $\reg > 0$, and let $\hyp_{\K,\Data} \gets \Algo_{\K}(\Data)$. Then, for any $\npt \geq 2$, any $\ntr \geq 1$, and any $\fail \in (0,1)$, with probability at least $1 - \fail$ over draws of $\X$,
\begin{equation}
\label{eq:stab_svm_risk_bound}
\Risk^{\1}(\hyp_{\K,\Data})
	\leq \Remp^{\hinge}(\hyp_{\K,\Data}) + \frac{2\rcnt\Knrm^2}{\reg\ntr} + \( \frac{2\Knrm^2}{\reg} + 1 \) \sqrt{\frac{\rcnt}{\ntr}\ln\frac{1}{\fail}}.
\end{equation}
\end{theorem}
\begin{proof}
Clearly, for $\conf=1$, $\Remp^{\hinge}$ dominates $\Remp^{\conf}$. We therefore apply \autoref{th:stab_ramp_risk_bound}, with $\conf=1$ and $\stab = \Knrm^2 / (2\reg\ntr)$.
\end{proof}

\subsection{Discussion of Bounds}
\label{sec:discussion_of_bounds}
Our bounds are dominated by the term $\rcnt / \ntr$, where $\rcnt$ is the maximum instance frequency (equivalently, the maximum degree in $\Graph$) and $\ntr$ is the number of examples (edges). We refer to the inverse of this ratio as the \define{effective training size}. Letting $\rcnt_i$ denote the frequency (i.e., degree) of instance $\x_i$, we have that
\begin{equation*}
\frac{\rcnt}{\ntr} = \frac{2\rcnt}{\sum_i^{\npt} \rcnt_i} = \frac{2}{\sum_i^{\npt} \rcnt_i / \rcnt},
\end{equation*}
It is straightforward to show that this quantity is minimized when $\Graph$ is regular---that is, when $\rcnt_i$ is uniform. In fact, for any regular $\Graph$, we have that $\rcnt / \ntr = 2 / \npt$.

Assuming one cannot acquire new examples, one can discard examples to make a regular graph, which gives the optimal ratio. This is equivalent to finding a $k$-regular spanning subgraph (i.e., $k$-factor), for some $k \geq 1$. This is not always possible without reducing the number of instances (i.e., vertices), as some graphs do not admit such a $k$-factor. For example, if the highest degree vertex is adjacent to multiple degree-1 vertices (as in a star graph), then certain vertices will be ``isolated" when edges are removed. In fact, the \emph{effective} number of instances $\npt'$ is the order of the largest $k$-regular induced subgraph, for some $k \geq 1$; we therefore have that the effective training size is upper bounded by $\npt' / 2$. That said, identifying $\npt'$ for an arbitrary graph is NP-hard.

It is tempting to think that, by discarding examples to induce a 1-regular subgraph, one can reduce our learning setup to the i.i.d.\ scenario and consequently apply classical analysis. 
However, there may be a regular (sub)graph of higher degree that yields a better effective training size. For instance, consider a graph consisting of $t$ disjoint triangles (i.e., $\npt = 3t$); this graph is already 2-regular, so without pruning edges it has an effective training size $\npt / 2 = 3t / 2$; if pruned to a 1-regular subgraph, the effective training size would be just $t$. Moreover, while the above shows that discarding examples might minimize our bounds, without intimate knowledge of the learning algorithm, hypothesis class or distribution, our bounds may be overly pessimistic in certain scenarios. Stronger assumptions may lead to tighter risk bounds, to support the intuition that more training data---albeit dependent data---will always improve generalization.

\section{On Subsampling and the Rate of Uniform Convergence}
\label{sec:subsampling}

We have shown that the empirical risk converges to the true risk at a rate of $\bigO(\sqrt{\rcnt / \ntr})$, depending primarily on the size of the training set and the maximum frequency of any instance. While $\ntr$ may be determined by one's annotation or computation budget, $\rcnt$ depends on the subsampling used to select the training set. In this section, we examine the relationship between the subsampling process used by the labeler and the rate of uniform convergence.

Recall that the labeler \emph{cannot} subsample based on the values of the input data, but it can subsample patterns that help or hurt generalization. If the labeler is working \emph{against} the learner, it can select pairs such that one instance appears in all training examples, meaning $\rcnt = \ntr$. In this scenario, our bounds indicate that a hypothesis learned from this training set might not generalize. In contrast, if the labeler is working \emph{with} the learner, it can subsample pairs so as to induce a regular label graph, as discussed in the previous section. This would yield an optimal convergence rate of $\bigO(1 / \sqrt{\npt})$, comparable to classical results.

We may also consider a setting in which the subsampling is a random process. For example, if the labeler selects pairs uniformly at random without replacement, then, as previously noted, this process can be modeled by the Erd\"{o}s-R\'{e}nyi random graph model. We then have that the rate of convergence is a function of the maximum degree of $\ERGraph(\npt,\ntr)$, since this is equivalent to the maximum instance frequency.
%
\begin{lemma}
\label{lem:mxdeg_erm}
Let $\Graph \defeq (\Nodes,\Edges)$ be a graph in $\ERGraph(\npt,\ntr)$, for a given $\npt$ and $\ntr$. Then, with probability at least $1 - \fail$, its maximum degree $\mxdeg(\Graph)$ is upper bounded as
\begin{equation}
\label{eq:mxdeg_erm}
\mxdeg(\Graph) \leq \frac{2\ntr}{\npt} \(1 + \sqrt{ \frac{3\npt}{2\ntr} \ln\frac{\npt}{\fail} }\).
\end{equation}
\end{lemma}
We provide the proof in \autoref{sec:proof_mxdeg_erm}.

Using this as an upper bound for $\rcnt$ (since $\rcnt = \mxdeg(\Graph)$), we obtain the following corollary of \autoref{th:rad_kernel_risk_bound}.
A similar result can be shown for \autoref{th:stab_svm_risk_bound}.
%
\begin{theorem}
\label{th:rad_kernel_risk_bound_erm}
Let $\X \in \XS^{\npt}$, $\Data \gets \Oracle_{\ntr}(\X)$, $\hyp_{\K,\Data}$, $\Knrm$ and $\reg$ be as previously defined. If $\Oracle$ samples $\ntr$ examples uniformly at random from $\X^2$, then, for any $\npt \geq 2$, any $\ntr \geq \npt/2$, any $\conf > 0$, and any $\fail \in (0,1)$, with probability at least $1 - \fail$ over draws of $\X$, and $\ERMconst \defeq 1 + \sqrt{ \frac{3\npt}{2\ntr} \ln\frac{2\npt}{\fail} }$,
\begin{equation}
\label{eq:rad_kernel_risk_bound_erm}
\Risk^{\1}(\hyp_{\K,\Data})
	\leq \Remp^{\conf}(\hyp_{\K,\Data})
	+ \frac{\sqrt{32}\Knrm}{\conf\sqrt{\npt}}
	+ \sqrt{\frac{\ERMconst+1}{\npt}\ln\frac{2}{\fail}}.
\end{equation}
\end{theorem}
\begin{proof}
\autoref{eq:rad_kernel_risk_bound_erm} follows from \autoref{th:rad_kernel_risk_bound} and \autoref{lem:mxdeg_erm}, by allowing failure probability $\fail/2$ to \autoref{eq:rad_kernel_risk_bound} and $\fail/2$ to \autoref{eq:mxdeg_erm}. We then simplify the bound by leveraging the fact that $1/\ntr \leq 2/\npt$. 
\end{proof}
We point out that these bounds have a natural interpretation. Whereas \citet{agarwal:jmlr09} invoke parameterized families of edge distributions, we consider a simple, intuitive learning setup in which the only parameter is the size $\ntr$ of the training set. If 
$\ntr \geq \npt\ln\npt$, then $\ERMconst = \bigO(1)$, and we obtain a uniform convergence rate of $\bigO(1 / \sqrt{\npt})$. For $\ntr \geq \npt/2$, we have that $\ERMconst = \bigO(\sqrt{\ln\npt})$, and the rate is still $\tilde\bigO(1/\sqrt{\npt})$.

\subsection*{Acknowledgements}
This work was partially supported by NSF CAREER grant 0746930 and NSF grant IIS1218488.

\break

\bibliographystyle{plainnat}
\bibliography{br_risk_bounds}

\begin{thebibliography}{19}
\providecommand{\natexlab}[1]{#1}
\providecommand{\url}[1]{\texttt{#1}}
\expandafter\ifx\csname urlstyle\endcsname\relax
  \providecommand{\doi}[1]{doi: #1}\else
  \providecommand{\doi}{doi: \begingroup \urlstyle{rm}\Url}\fi

\bibitem[Agarwal and Niyogi(2009)]{agarwal:jmlr09}
S.~Agarwal and P.~Niyogi.
\newblock Generalization bounds for ranking algorithms via algorithmic
  stability.
\newblock \emph{Journal of Machine Learning Research}, 10:\penalty0 441--474,
  2009.

\bibitem[Bar-{H}illel and Weinshall(2003)]{barhillel:colt03}
A.~Bar-{H}illel and D.~Weinshall.
\newblock Learning with equivalence constraints, and the relation to multiclass
  learning.
\newblock In \emph{Proceedings of the Conference on Computational Learning
  Theory}, 2003.

\bibitem[Bartlett and Mendelson(2003)]{bartlett:jmlr03}
P.~Bartlett and S.~Mendelson.
\newblock Rademacher and gaussian complexities: risk bounds and structural
  results.
\newblock \emph{Journal of Machine Learning Research}, 3:\penalty0 463--482,
  March 2003.

\bibitem[Bousquet and Elisseeff(2002)]{bousquet:jmlr02}
O.~Bousquet and A.~Elisseeff.
\newblock Stability and generalization.
\newblock \emph{Journal of Machine Learning Research}, 2:\penalty0 499--526,
  March 2002.

\bibitem[Cl\'{e}men\c{c}on et~al.(2008)Cl\'{e}men\c{c}on, Lugosi, and
  Vayatis]{clemencon:as08}
S.~Cl\'{e}men\c{c}on, G.~Lugosi, and N.~Vayatis.
\newblock Ranking and empirical minimization of {U}-statistics.
\newblock \emph{The Annals of Statistics}, 36\penalty0 (2):\penalty0 844--874,
  April 2008.
\newblock ISSN 0090-5364.

\bibitem[Cristianini and Shawe-Taylor(2000)]{cristianini:svms}
N.~Cristianini and J.~Shawe-Taylor.
\newblock \emph{An introduction to support vector machines and other
  kernel-based learning methods}.
\newblock Cambridge University Press, March 2000.

\bibitem[Erd\"{o}s and Lov\'{a}sz(1975)]{erdos:lll75}
P.~Erd\"{o}s and L.~Lov\'{a}sz.
\newblock Problems and results on 3-chromatic hypergraphs and some related
  questions.
\newblock In \emph{Infinite and Finite Sets (to Paul Erd\"{o}s on his 60th
  birthday)}, Colloquia Mathematica Societatis J\'{a}nos Bolyai, pages
  609--627. J\'{a}nos Bolyai Mathematical Society, 1975.

\bibitem[Goldman et~al.(1993)Goldman, Schapire, and Rivest]{goldman:siam93}
S.~Goldman, R.~Schapire, and R.~Rivest.
\newblock Learning binary relations and total orders.
\newblock \emph{SIAM J. Computing}, 22:\penalty0 46--51, 1993.

\bibitem[Janson(2004)]{janson:rsa04}
S.~Janson.
\newblock Large deviations for sums of partly dependent random variables.
\newblock \emph{Random Structures Algorithms}, 24:\penalty0 234--248, 2004.

\bibitem[Jin et~al.(2009)Jin, Wang, and Zhou]{jin:nips09}
R.~Jin, S.~Wang, and Y.~Zhou.
\newblock Regularized distance metric learning: Theory and algorithm.
\newblock In \emph{Advances in Neural Information Processing Systems 22}, pages
  862--870. 2009.

\bibitem[Koltchinskii and Panchenko(2002)]{koltchinskii:as02}
V.~Koltchinskii and D.~Panchenko.
\newblock Empirical margin distributions and bounding the generalization error
  of combined classifiers.
\newblock \emph{Annals of Statistics}, 30:\penalty0 1--50, 2002.

\bibitem[Ledoux and Talagrand(1991)]{ledoux:probability91}
M.~Ledoux and M.~Talagrand.
\newblock \emph{Probability in Banach spaces: isoperimetry and processes}.
\newblock Ergebnisse der Mathematik und ihrer Grenzgebiete. Springer-Verlag,
  1991.

\bibitem[McDiarmid(1989)]{mcdiarmid:sc89}
C.~McDiarmid.
\newblock On the method of bounded differences.
\newblock In \emph{Surveys in Combinatorics}, volume 141 of \emph{London
  Mathematical Society Lecture Note Series}, pages 148--188. Cambridge
  University Press, 1989.

\bibitem[Mohri and Rostamizadeh(2009)]{mohri:nips08}
M.~Mohri and A.~Rostamizadeh.
\newblock Rademacher complexity bounds for non-i.i.d. processes.
\newblock In \emph{Advances in Neural Information Processing Systems 21}, pages
  1097--1104. MIT Press, 2009.

\bibitem[Mohri and Rostamizadeh(2010)]{mohri:jmlr10}
M.~Mohri and A.~Rostamizadeh.
\newblock Stability bounds for stationary $\varphi$-mixing and $\beta$-mixing
  processes.
\newblock \emph{Journal of Machine Learning Research}, 11:\penalty0 789--814,
  2010.

\bibitem[Ralaivola et~al.(2010)Ralaivola, Szafranski, and
  Stempfel]{ralaivola:jmlr10}
L.~Ralaivola, M.~Szafranski, and G.~Stempfel.
\newblock Chromatic {PAC}-bayes bounds for non-iid data: Applications to
  ranking and stationary $\beta$-mixing processes.
\newblock \emph{Journal of Machine Learning Research}, 11:\penalty0 1927--1956,
  2010.

\bibitem[Usunier et~al.(2006)Usunier, Amini, and Gallinari]{usunier:nips05}
N.~Usunier, M.-R. Amini, and P.~Gallinari.
\newblock Generalization error bounds for classifiers trained with
  interdependent data.
\newblock In \emph{Advances in Neural Information Processing Systems 18}, pages
  1369--1376. MIT Press, Cambridge, MA, 2006.

\bibitem[Valiant(1984)]{valiant:stoc84}
L.~Valiant.
\newblock A theory of the learnable.
\newblock In \emph{Proceedings of the Sixteenth Annual ACM Symposium on Theory
  of Computing}, pages 436--445, New York, NY, USA, 1984. ACM.

\bibitem[Vizing(1964)]{vizing:coloring64}
V.~Vizing.
\newblock On an estimate of the chromatic class of a p-graph.
\newblock \emph{Diskret. Analiz.}, 3:\penalty0 25--30, 1964.

\end{thebibliography}

\newpage

\appendix
\section{Proof of \autoref{lem:mxdeg_erm}}
\label{sec:proof_mxdeg_erm}

For each pair of vertices $\{\node_i,\node_j\}$, define a random variable $\Edges_{i,j} \defeq \1[\{i,j\} \in \Edges]$. Note that $\deg(\node_i) = \sum_{j \neq i} \Edges_{i,j}$ and, via linearity of expectation,
\begin{equation*}
\Ep[\deg(\node_i)]
	= \sum_{j \neq i} \Ep[ \Edges_{i,j} ]
	= (\npt-1) \cdot \frac{ { {\npt\choose2}-1 \choose \ntr-1} }{ { {\npt\choose2} \choose \ntr} } 
	= \frac{2\ntr}{\npt}.
\end{equation*}
Since the expected degree is uniform, let $\avgdeg \defeq \Ep[\deg(\node_i)]$. Using the union bound, for any $t > 0$, we have that
\begin{align*}
\Pr[ \mxdeg(\Graph) \geq t ]
	&\leq \sum_{i=1}^{\npt} \Pr[ \deg(\node_i) \geq t ] \\
	&= \sum_{i=1}^{\npt} \Pr\left[ \sum_{j \neq i} \Edges_{i,j} \geq t \right].
\end{align*}
Though the variables $\{ \Edges_{i,j} \}_{j \neq i}$ are dependent, it is straightforward to show that they are negatively correlated. Therefore, we can apply  multiplicative Chernoff bounds, with $t = \avgdeg(1+\eps)$, to obtain
\begin{equation*}
\Pr[ \mxdeg(\Graph) \geq t ]
	\leq \sum_{i=1}^{\npt} \exp(-\avgdeg\eps^2/3)
	= \npt \exp\(\frac{-2\ntr\eps^2}{3\npt}\).
\end{equation*}
Setting this equal to $\fail$, then solving for $\eps$ and $t$ completes the proof.

\end{document}